\def\year{2022}\relax
\documentclass[letterpaper]{article} 
\usepackage{aaai21}  
\usepackage{times}  
\usepackage{helvet}  
\usepackage{courier}  
\usepackage[hyphens]{url}  
\usepackage{graphicx} 
\usepackage{natbib}  
\usepackage{caption} 
\DeclareCaptionStyle{ruled}{labelfont=normalfont,labelsep=colon,strut=off} 
\usepackage{comment}
\usepackage{todonotes}
\usepackage{amsthm}
\usepackage{amsmath,amssymb} 
\usepackage{algorithm}
\usepackage[noend]{algpseudocode}
\usepackage{booktabs} 

\newcommand{\mnmz}{\operatorname*{minimize}}
\newcommand{\st}{\text{subject to}}
\newcommand{\eps}{\varepsilon}
\newcommand{\dist}{\operatorname{distance}}

\newcommand{\xzero}{x_{0}}
\newcommand{\zzero}{z_{0}}
\newcommand{\pinit}{p_{\rm init}}
\newcommand{\pnext}{p_{\rm next}}

\newtheorem{theorem}{Theorem}

\title{Adversarial examples by perturbing high-level features \\ in intermediate decoder layers}
\author{Vojt\v{e}ch \v{C}erm\'ak\textsuperscript{\rm 1}, Luk\'a\v{s} Adam\textsuperscript{\rm 1} \\}

\affiliations{
    \textsuperscript{\rm 1}Faculty of Electrical Engineering \\
    Czech Technical University in Prague \\
    Technická 2, Prague, Czech Republic\\
    adamluk3@fel.cvut.cz
}

\begin{document}

\maketitle

\begin{abstract}
We propose a novel method for creating adversarial examples. Instead of perturbing pixels, we use an encoder-decoder representation of the input image and perturb intermediate layers in the decoder. This changes the high-level features provided by the generative model. Therefore, our perturbation possesses semantic meaning, such as a longer beak or green tints. We formulate this task as an optimization problem by minimizing the Wasserstein distance between the adversarial and initial images under a misclassification constraint. We employ the projected gradient method with a simple inexact projection. Due to the projection, all iterations are feasible, and our method always generates adversarial images. We perform numerical experiments on the MNIST and ImageNet datasets in both targeted and untargeted settings. We demonstrate that our adversarial images are much less vulnerable to steganographic defence techniques than pixel-based attacks. Moreover, we show that our method modifies key features such as edges and that defence techniques based on adversarial training are vulnerable to our attacks.
\end{abstract}

\section{Introduction}

In the past decade, the widespread application of deep neural networks raised security concerns as it creates incentives for attackers to exploit any potential weakness. For example, attackers could create traffic signs invisible to autonomous cars or make malware filters ignore threats. Those security concerns rose since \cite{szegedy2014intriguing} showed that deep neural networks are vulnerable to small perturbations of inputs that are designed to cause misclassification of a classifier. These adversarial examples are indistinguishable from natural examples as the perturbation is too small to be perceived by humans. 

We extend the usual approach to constructing adversarial examples that focuses on norm-bounded pixel modifications. Instead of perturbing the pixels directly, we perturb features learned by an encoder-decoder model. There are several ways to perturb the features from the decoder, ranging from high-level features collected from initial decoder layers to much finer features collected from decoder layers near the reconstructed image. When we perturb features in the initial layers near the latent image representation, even small perturbations may completely change the image meaning. On the other hand, perturbations on fine features are very close to pixel perturbations and carry little semantic information. As a compromise, we suggest perturbing intermediate decoder layers.

\begin{figure}[!ht]
\centering
\begin{tikzpicture}
  \def\tikzHA{1.5};
  \def\tikzHB{0.75};
  \def\tikzHC{1};  
  \def\tikzWA{0.4};
  \def\tikzWB{1.5};
  \def\tikzWC{0.4};
  \def\tikzS{0.15};
  \def\tikzWFA{0};
  \def\tikzWFB{\tikzWA};
  \def\tikzWFC{\tikzS+\tikzWA};
  \def\tikzWFD{\tikzS+\tikzWA+\tikzWB};
  \def\tikzWFE{2*\tikzS+\tikzWA+\tikzWB};
  \def\tikzWFF{2*\tikzS+2*\tikzWA+\tikzWB};
  \def\tikzWFG{3*\tikzS+2*\tikzWA+\tikzWB};
  \def\tikzWFH{3*\tikzS+2*\tikzWA+2*\tikzWB};
  \def\tikzWFI{4*\tikzS+2*\tikzWA+2*\tikzWB};
  \def\tikzWFJ{4*\tikzS+3*\tikzWA+2*\tikzWB};
  \def\tikzWFK{5*\tikzS+3*\tikzWA+2*\tikzWB};
  \def\tikzWFL{5*\tikzS+4*\tikzWA+2*\tikzWB};
  \def\tikzWFM{6*\tikzS+4*\tikzWA+2*\tikzWB};
  \def\tikzWFN{6*\tikzS+4*\tikzWA+3*\tikzWB};
  \def\tikzWFO{7*\tikzS+4*\tikzWA+3*\tikzWB};
  \def\tikzWFP{7*\tikzS+5*\tikzWA+3*\tikzWB};
  \def\tikzTA{\tikzHA+0.38};
  \def\tikzTB{-\tikzHA-0.35};  
  \draw [fill=white] (\tikzWFA,-\tikzHA) -- (\tikzWFB,-\tikzHA) -- (\tikzWFB,\tikzHA) --  (\tikzWFA,\tikzHA) -- cycle;
  \draw [fill=gray!20] (\tikzWFC,-\tikzHA) -- (\tikzWFD,-\tikzHB) -- (\tikzWFD,\tikzHB) --  (\tikzWFC,\tikzHA) -- cycle;
  \draw [fill=white] (\tikzWFE,-\tikzHB) -- (\tikzWFF,-\tikzHB) -- (\tikzWFF,\tikzHB) --  (\tikzWFE,\tikzHB) -- cycle;
  \draw [fill=gray!20] (\tikzWFG,-\tikzHB) -- (\tikzWFH,-\tikzHC) -- (\tikzWFH,\tikzHC) --  (\tikzWFG,\tikzHB) -- cycle;
  \draw [fill=white] (\tikzWFI,-\tikzHC) -- (\tikzWFJ,-\tikzHC) -- (\tikzWFJ,\tikzHC) --  (\tikzWFI,\tikzHC) -- cycle;
  \draw [fill=green!70] (\tikzWFK,-\tikzHC) -- (\tikzWFL,-\tikzHC) -- (\tikzWFL,\tikzHC) --  (\tikzWFK,\tikzHC) -- cycle;
  \draw [fill=gray!20] (\tikzWFM,-\tikzHC) -- (\tikzWFN,-\tikzHA) -- (\tikzWFN,\tikzHA) --  (\tikzWFM,\tikzHC) -- cycle;
  \draw [fill=white] (\tikzWFO,-\tikzHA) -- (\tikzWFP,-\tikzHA) -- (\tikzWFP,\tikzHA) --  (\tikzWFO,\tikzHA) -- cycle;
  \node at ({0.5*(\tikzWFA+\tikzWFB)},\tikzTB) {Source};
  \node at ({0.5*(\tikzWFE+\tikzWFF)},\tikzTB) {Latent};
  \node at ({0.5*(\tikzWFJ+\tikzWFK)},\tikzTB) {Intermediate};
  \node at ({0.5*(\tikzWFO+\tikzWFP)-0.25},\tikzTB) {Reconstructed};  
  \node at ({0.5*(\tikzWFC+\tikzWFD)},0) {Encoder};
  \node at ({0.5*(\tikzWFG+\tikzWFH)},0) {Decoder$_1$};
  \node at ({0.5*(\tikzWFM+\tikzWFN)},0) {Decoder$_2$};
  \node [draw] at ({0.5*(\tikzWFI+\tikzWFJ)-2.5},\tikzTA) {intermediate latent representation};
  \node [draw] at ({0.5*(\tikzWFK+\tikzWFL)+1.1},\tikzTA) {inserted perturbation};
  \draw [->] ({0.5*(\tikzWFI+\tikzWFJ)-2.5},\tikzTA-0.28) -- ({0.5*(\tikzWFI+\tikzWFJ)},\tikzHC+0.1);
  \draw [->] ({0.5*(\tikzWFK+\tikzWFL)+1.1},\tikzTA-0.28) -- ({0.5*(\tikzWFK+\tikzWFL)},\tikzHC+0.1);
\end{tikzpicture}
\end{figure}

Our way of generating adversarial images provides several advantages:
\begin{itemize}
 \item The perturbations are interpretable. They modify high-level features such as fur colour or beak length.
 \item The perturbations often follow edges. This means that they are less detectable by steganographic defence techniques \cite{johnson1998exploring}.
 \item The perturbations keep the structure of the unperturbed images, including hidden structures unrelated to their semantic meaning.
 \item The decoder ensures that the adversarial images have the correct pixel values; there is no need to project to $[0,1]$.
\end{itemize}
The drawback of our approach is that we cannot perturb an arbitrary image but only those representable by the decoder.

To obtain a formal optimization problem, we minimize the distance between the original and adversarial images in the \textit{reconstructed} space. To successfully generate an adversarial image, we add a constraint requiring that the perturbed image is misclassified. To handle this constraint numerically, we propose to use the projected gradient method. We approximate the projection operator by a fast method which always returns a feasible point. This brings additional benefits to our method:
\begin{itemize}
 \item The method works with feasible points. Even if it does not converge, it produces an adversarial image.
 \item Because we minimize the distance between the original and adversarial images, we do not need to specify their maximal possible distance as many methods do.
\end{itemize}

Numerical experiments present results on MNIST and ImageNet in both targeted and untargeted settings. They support all advantages mentioned above. We use both the $l_2$ and Wasserstein distances and show in which settings each performs better. We select a simple steganography defence mechanism and show that our attack is much less vulnerable to it than pixel-based perturbations. We examine the difference between perturbing the latent and intermediate layers and show that the intermediate layers indeed modify the high-level features of the reconstructed image. Finally, we show how our algorithm gradually incorporates the high-level features from the initial to the adversarial image. Our codes are available online to promote reproducibility.\footnote{https://anonymous.4open.science/r/latent-adv-examples-5C59}

\subsection{Related work}

The threat model based on small $l_p$ norm-bounded perturbations has been the main focus of research. It was originally introduced in \cite{szegedy2014intriguing}, where they used L-BFGS to minimize the $l_2$ distance between the original and adversarial images. Since then, many new ways how to both attack and defend against adversarial examples have been introduced. \cite{goodfellow2015explaining} introduced Fast Gradient Sign Method (FGSM), which is bounded by the $l_{\infty}$ metric. The authors used FGSM to generate new adversarial examples and used them to augment the training set in the adversarial training defence technique. A natural way to extend the FGSM attack is to iterate the gradient step as it is done in the Basic Iterative Method of \cite{kurakin2017adversarial} and Projected Gradient Descend attack of \cite{madry2019deep}. \cite{papernot2015limitations} argued to use the $l_0$ distance to model human perception and proposed a class of attacks optimized under $l_0$ distance. \cite{carlini2017evaluating} designed strong attack algorithms based on optimizing the $l_{0}$, $l_{2}$ and $l_{\infty}$ distances.
 
\cite{xiao2019generating} used GANs to generate noise for adversarial perturbation. Other authors used generative models in defence against adversarial examples. MagNet of \cite{meng2017magnet} used reconstruction error of variational autoencoders to detect adversarial examples. A similar idea was used in DefenceGAN of \cite{samangouei2018defensegan}, where they cleaned adversarial images by matching them to their representation in latent space of the GAN trained on clean data.
 
We use similar optimization techniques as in decision-based attacks such as Boundary attack of \cite{brendel2018decisionbased} and 
HopSkipJump attack of \cite{chen2020hopskipjumpattack}. The optimization techniques always keep intermediate results in the adversarial region and always outputs misclassified examples.
 
Some works have already investigated adversarial examples outside of the $l_p$ norm. \cite{wong2020wasserstein} used the Sinkhorn approximation of the Wasserstein distance to create adversarial images with the same geometrical structure as the original images. The Wasserstein distance is more suitable than the standard $l_p$ distance metrics because it preserves differences in high-level features. We increase this benefit by additionally modifying the high-level features instead of pixels.
 
The Unrestricted Adversarial Examples of \cite{song2018constructing}, are adversarial examples constructed from scratch using conditional generative models. Our paper differs in several aspects. First, we use the intermediate instead of the latent representation to perturb high-level features. Second, we use an unconditional generator, which allows us to freely move in the intermediate latent space and perform several operations such as the projection.

\section{Proposed Formulation}

Finding an adversarial image amounts to finding some image $x$ which is close to a given $x_0$, and the neural network misclassifies it. For reasons mentioned in the introduction, we do not work with the original images but with their latent representations. Therefore, we need a decoder $D$ which maps the latent representation $z$ to the original representation $x$. Since we want to insert some perturbation $p$ to the intermediate layer in the decoder, we split the decoder into two parts $D = D_2\circ D_1$. We call $D_1(z)$ the intermediate latent representation and $D_2(D_1(z))$ the reconstructed image. 

The mathematical formulation of finding an adversarial image close to $x_0$ reads:
\begin{equation}\label{eq:problem}
    \aligned
    \mnmz_{p}\qquad &\dist(x, \xzero) \\
    \st\qquad &x = D_2(D_1(\zzero) + p), \\
    &\xzero = D_2(D_1(\zzero)), \\
    &g(x) \le 0, \\
    &x\in[0,1]^n.
    \endaligned
\end{equation}
Here, $z_0$ is the latent representation of the image $x_0$. We insert the perturbation $p$ to the intermediate latent representation $D_1(z_0)$ and only then reconstruct the image by applying the second part of the decoder $D_2$.

When $D_1$ is the identity, and $D_2$ is the decoder, we obtain the case when perturbations appear in the latent space. Similarly, when $D_1$ is the decoder, and $D_2$ is the identity, we recover the standard pixel perturbations. Therefore, our formulation \eqref{eq:problem} generalizes both approaches.

\subsection{Objective} The objective function measures the distance between the adversarial $x$ and the original $x_0$ image in the reconstructed space. We will use the $l_2$ norm and the Wasserstein distance. Having the distance in the objective is advantageous because we do not need to specify the $\eps$-neighborhood when this distance is in the constraint.

\subsection{Constraints}

Formulation \eqref{eq:problem} has multiple constraints. Constraint $x\in[0,1]^n$ says that the pixels need to stay in this range. Since $x$ is an output of the decoder, this constraint is always satisfied.

The other constraint $g(x)\le 0$ is a misclassification constraint. We employ the margin function
\begin{equation}\label{eq:margin}
m(x, k) = \underset{i, i \neq k}{\text{max }}F_i(x) - F_k(x),
\end{equation}
where $F = [F_{1}(x), ..., F_{m}(x)]$ is a classifier with $m$ classes and $k$ is a class index. We define this constraint for both targeted and untargeted attacks:
\begin{equation}
  \begin{aligned}
  \text{targeted}: \quad& g(x) = m(x, k), \\
  \text{untargeted}: \quad& g(x) = - m(x, \operatorname{argmax}_{i=1,\dots,m}F_i(x)).
  \end{aligned}
\end{equation}
For the targeted case, we need to specify the target class $k$, while for the untargeted case, the class $k$ is the classifier prediction. For the latter case, we need to multiply the margin by $-1$ as we require the misclassified images to satisfy $g(x) \le 0$.

We will later use that the original image always satisfies $g(x_0) > 0$ because otherwise the original image is already misclassified, and no perturbation $(p=0)$ is the optimal solution of \eqref{eq:problem}.

\section{Proposed Solution Method}

We propose to solve \eqref{eq:problem} by the projected gradient method \cite{nocedal2006numerical} with inexact projection. Our algorithm produces a feasible point at every iteration. Therefore, it always generates a misclassified image.

\subsection{Proposed algorithm}

Since we optimize with respect to $p$, we define the objective and constraint by
\begin{equation}\label{eq:fg}
\aligned
f(p) &= \dist(D_2(D_1(z_0) + p), x_{0}), \\
\hat g(p) &= g(D_2(D_1(z_0) + p)).
\endaligned
\end{equation}
We describe our procedure in Algorithm \ref{alg2}. First, we initialize $p$ by some feasible $\pinit$ and then run the projected gradient method for $\max_{\rm iter}$ iterations. Step \ref{alg2_step1} computes the optimization step by minimizing the objective $f$ and step \ref{alg2_step2} uses the inexact projection (described later) to project the suggested iteration $\pnext$ back to the feasible set. As we will see later, the constraint $\pnext\neq p$ implies that $\pnext$ was not feasible, and it was projected onto the boundary. In such a case, steps \ref{alg2_if}-\ref{alg2_if2} ``bounce away'' from the boundary. Since $-\nabla \hat g(p)$ points inside the feasible set, step~\ref{alg2_while} finds some $\beta>0$ such that $p-\beta\nabla \hat g(p)$ lies in the interior of the feasible set.

\begin{algorithm}
    \caption{for finding adversarial images by solving \eqref{eq:problem}}
    \label{alg2}
    \begin{algorithmic}[1]
    \State $p \gets \pinit$

    \For{$i \in \{0,\dots,\max_{\rm iter}\}$}
        \If{$i>0$ \textbf{and} $\pnext \neq p$} 
        \State $\beta \gets \beta_i$ \label{alg2_if}
        \While{ $\hat g(p - \beta \nabla \hat g(p)) \geq 0$} \label{alg2_while}
            \State $\beta \gets \frac{\beta}{2} $
        
        \EndWhile
        \State $p \gets p - \beta \nabla \hat g(p) $ \label{alg2_if2}
        \EndIf
        
        \State $\pnext \gets p - \alpha_i \nabla  f(p) $ \label{alg2_step1}
        \State $p \gets \textsc{Project}(\pnext, p, \delta) $ \label{alg2_step2}
    \EndFor
    \State \textbf{return} $p$
\end{algorithmic}
\end{algorithm}

\subsection{Inexact projection}

Step \ref{alg2_step2} in Algorithm \ref{alg2} uses the inexact projection. We summarize this projection in Algorithm \ref{alg1}. Its main idea is to find a feasible point on the line between $p$ and $\pnext$. This line is parameterized by $(1-c)p + c\pnext$ for $c\in[0,1]$. Due to the construction of Algorithm \ref{alg2}, $p$ is always strictly feasible and therefore $\hat g(p) < 0$. If $\hat g(\pnext) < 0$, then $\pnext$ is feasible and we accept it in step \ref{alg1_pnext}. In the opposite case, we have $\hat g(p) < 0$ and $\hat g(\pnext)\ge 0$ and we can use the bisection method to find some $c\in (0,1)$ for which the constraint is satisfied. The standard bisection method would return a point with $\hat g((1-c)p+c\pnext) = 0$, however, since the formulation \eqref{eq:problem} contains the inequality constraint, we require the constraint value to lie only in some interval $[-\delta,0]$ instead of being zero.

\begin{algorithm}
    \caption{Inexact projection onto the feasible set}
    \label{alg1}
    \begin{algorithmic}[1]
    \Procedure{Project}{$\pnext, p, \delta$}
    \State \textbf{assert} $\hat g(p) < 0$
    \If{$\hat g(\pnext) < 0$}
        \State \textbf{return} $\pnext$ \label{alg1_pnext}
    \EndIf
    \State $a\gets 0$, $b\gets 1$, $c\gets \frac12(a+b)$
    \While{$\hat g((1-c)p + c\pnext) \notin [-\delta, 0]$}
        \If{$\hat g((1-c)p + c\pnext) > 0$}
            \State $b\gets c$
        \Else
            \State $a\gets c$
        \EndIf
        \State $c \gets \frac12(a+b)$
    \EndWhile
    \State \textbf{return} $(1-c) p + c\pnext$
    \EndProcedure
    \end{algorithmic}
\end{algorithm}

Figure \ref{fig:proj} depicts the projection in the \textit{intermediate latent space}. The light-grey region is feasible, and the white region is infeasible. The infeasible region always contains $D_1(z_0)$, while the feasible region always contains $D_1(z_0)+p$. If $D_1(z_0)+\pnext$ lies in the feasible region, the projection returns $\pnext$. In the opposite case, we look for some point on the line between $D_1(z_0)+p$ and $D_1(z_0)+\pnext$. The points which may be accepted are depicted by the thick solid line. The length of this line is governed by the threshold $\delta$. The extremal case $\delta=0$ always returns the point on the boundary, while $\delta=\infty$ prolongs the line to $D_1(z_0)+p$.

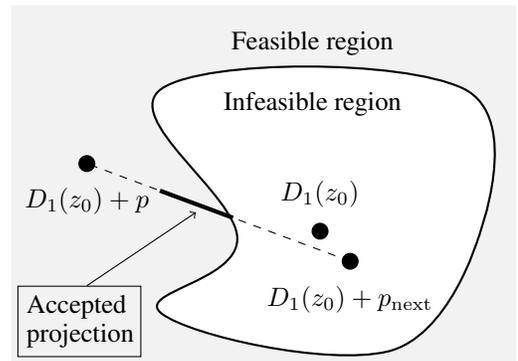
\begin{figure}[!ht]
\caption{Inexact projection in the intermediate latent space. The points which may be accepted are depicted by the thick solid line.}
\label{fig:proj}
\centering
\begin{tikzpicture}
  \def\tikzA{0.555};
  \def\tikzB{0.2775};
  \def\tikzBoxA{2.7};
  \def\tikzBoxB{3.1};
  \def\tikzBoxC{-4};
  \def\tikzBoxD{-1.7};
  \coordinate  (O) at (0,0);
  \filldraw [gray!10, fill=gray!10] plot coordinates {(\tikzBoxA,\tikzBoxB) (\tikzBoxC,\tikzBoxB) (\tikzBoxC,\tikzBoxD) (\tikzBoxA,\tikzBoxD) (\tikzBoxA,\tikzBoxB)};
  \filldraw [thick, fill=white]  plot[smooth, tension=.7] coordinates {(2.2,0) (2,2) (-2,2) (-1,0) (-2,-1) (1,-1.5) (2.2,0)};
  \node at (0,2.6) {Feasible region};
  \node at (0,1.8) {Infeasible region};
  \filldraw (0.5,-0.3) circle (3pt) node [yshift=-0.5cm] {$D_1(z_0) + \pnext$};
  \filldraw (-3,1) circle (3pt) node [yshift=-0.5cm] {$D_1(z_0) + p$};
  \filldraw (0.1,0.1) circle (3pt) node [yshift=0.5cm] {$D_1(z_0)$};
  \draw [dashed] (0.5,-0.3) -- (-3,1);
  \draw [ultra thick] ({\tikzA*0.5+(1-\tikzA)*(-3)},{\tikzA*(-0.3)+(1-\tikzA)*(1)}) -- ({\tikzB*0.5+(1-\tikzB)*(-3)},{\tikzB*(-0.3)+(1-\tikzB)*(1)});
  \node [draw,text width=1.4cm] at (-3.1,-1.1) {Accepted \\projection};  
  \draw [->] (-3.0,-0.625) -- (-1.55,0.35);
\end{tikzpicture}
\end{figure}

\subsection{Analysis of the algorithm}

The preceding text mentioned that Algorithm \ref{alg2} generates a strictly feasible point $p$, thus $g(D_2(D_1(z_0)+p)) < 0$. We prove this in the next theorem and add a speed of convergence. We postpone the proof to the appendix.

\begin{theorem}\label{thm:conv}
    Assume that $g\circ D_2$ is a continuous function. Then Algorithm \ref{alg1} generates a strictly feasible point $p$ in a finite number of iterations. If $g\circ D_2$ is moreover Lipschitz continuous with modulus $L$, then Algorithm \ref{alg1} converges in at most $\log_2 \frac{L\|\pnext-p\|}{\delta}$ iterations.
\end{theorem}

The previous theorem implies that all iterations $p$ produced by Algorithm \ref{alg2} are strictly feasible.

\section{Numerical Experiments}

This section shows numerical experiments on the image datasets MNIST \cite{lecun2010mnist} and ImageNet \cite{deng2009imagenet}.

\subsection{Used architectures}

As the classifiers to be fooled in our MNIST experiments, we train a non-robust classifier with architecture based on VGG blocks \cite{simonyan2015deep} and use the publicly available robust model \cite{madry2019deep}. We generate new digits using an unconditional ALI generator  \cite{donahue2017adversarial,dumoulin2017adversarially}. For ImageNet we use EfficientNet B0 \cite{tan2020efficientnet} as a classifier and BigBiGAN \cite{donahue2019large} as an encoder-decoder model.

\subsection{Numerical setting}

As an objective we use the standard $l_2$ distance and the Sinkhorn approximation \cite{cuturi2013sinkhorn} to the Wasserstein distance. This approximation adds a weighted Kullback-Leibler divergence to solve
\begin{equation}\label{eq:sinkhorn}
    \aligned
    \mnmz_{W}\qquad &\langle C,W\rangle + \lambda \langle W, \log W\rangle \\
    \st\qquad &W1 = x_0, W^\top 1=x, W\ge 0.
    \endaligned
\end{equation}
The optimal value of this problem is the Sinkhorn distance between images $x_0$ and $x$. The matrix $C$ specifies the distance between pixels. Since $x_0$ and $x$ are required to be probability distributions, we normalize the pixels from decoder output to sum to one. We do not need to impose the non-negativity constraint because the decoder outputs positive pixel intensities. We use the implementation from the GeomLoss package \cite{feydy2019interpolating}. Since the number of elements $W$ from \eqref{eq:sinkhorn} equals the number of pixels squared, using the Sinkhorn distance was infeasible for ImageNet, where we use only the $l_2$ distance.

For MNIST, we randomly generate the latent images $z_0$ and then use the decoder for reconstructed images $x_0$. We required that the classifier predicts the digit into the correct class with the probability of at least 0.99. For ImageNet, this technique produces images of lower quality, and we fed the encoder with real images and used their encoded representation.

We run all experiments for 1000 iterations. Algorithm \ref{alg2} requires stepsizes $\alpha$ and $\beta$, while Algorithm \ref{alg1} requires the threshold $\delta$. We found that the stepsize $\alpha$ is not crucial because even if it is large, the projection will reduce it. We therefore selected $\alpha=1$. For the second stepsize $\beta$, we selected a simple annealing scheme with exponential decay. In both cases, we used the normed gradient. For the threshold we selected $\delta=1$. Since the margin function \eqref{eq:margin} is bounded by $1$, Figure~\ref{fig:proj} then implies that the acceptable projection increases all the way to $D_1(z_0)+p$. Since at least one iteration of the projection is performed, the distance to the boundary at least halves. We found this to be a good compromise between speed and approximative quality. Theorem \ref{thm:conv} then implies that Algorithm \ref{alg2} converges in at most $\log_2 L$ iterations.

We run our experiments on an Nvidia GPU with 6GB of VRAM.

\subsection{Numerical results for MNIST}

\begin{table*}[!ht]
\centering

\begin{tabular}{@{}llllllll@{}}
\toprule
       Network & Attack          & \multicolumn{3}{c}{$l_2$ distance} & \multicolumn{3}{c}{Wasserstein distance} \\ \cmidrule(lr){3-5} \cmidrule(lr){6-8}
       &          &     CW attack & $l_2$ attack & Wasserstein  &              CW  attack & $l_2$ attack & Wasserstein \\
\midrule
Non-robust & Untargeted &             7.85 &     15.28 &              20.87 &                      1.77 &      1.49 &               0.08 \\
       & Targeted &            10.07 &     16.78 &              24.01 &                      2.76 &      1.83 &               0.12 \\
Robust & Untargeted &            17.36 &     22.78 &              24.68 &                      0.78 &      2.48 &               0.19 \\
       & Targeted &            24.49 &     27.05 &              29.75 &                      1.72 &      3.12 &               0.32 \\
\bottomrule
\end{tabular}

\caption{Mean $l_2$ and Wasserstein distances between original and adversarial images over 90 randomly selected images.}
\label{table:results}

\end{table*}

\begin{figure*}[!ht]
    \includegraphics[width=\textwidth]{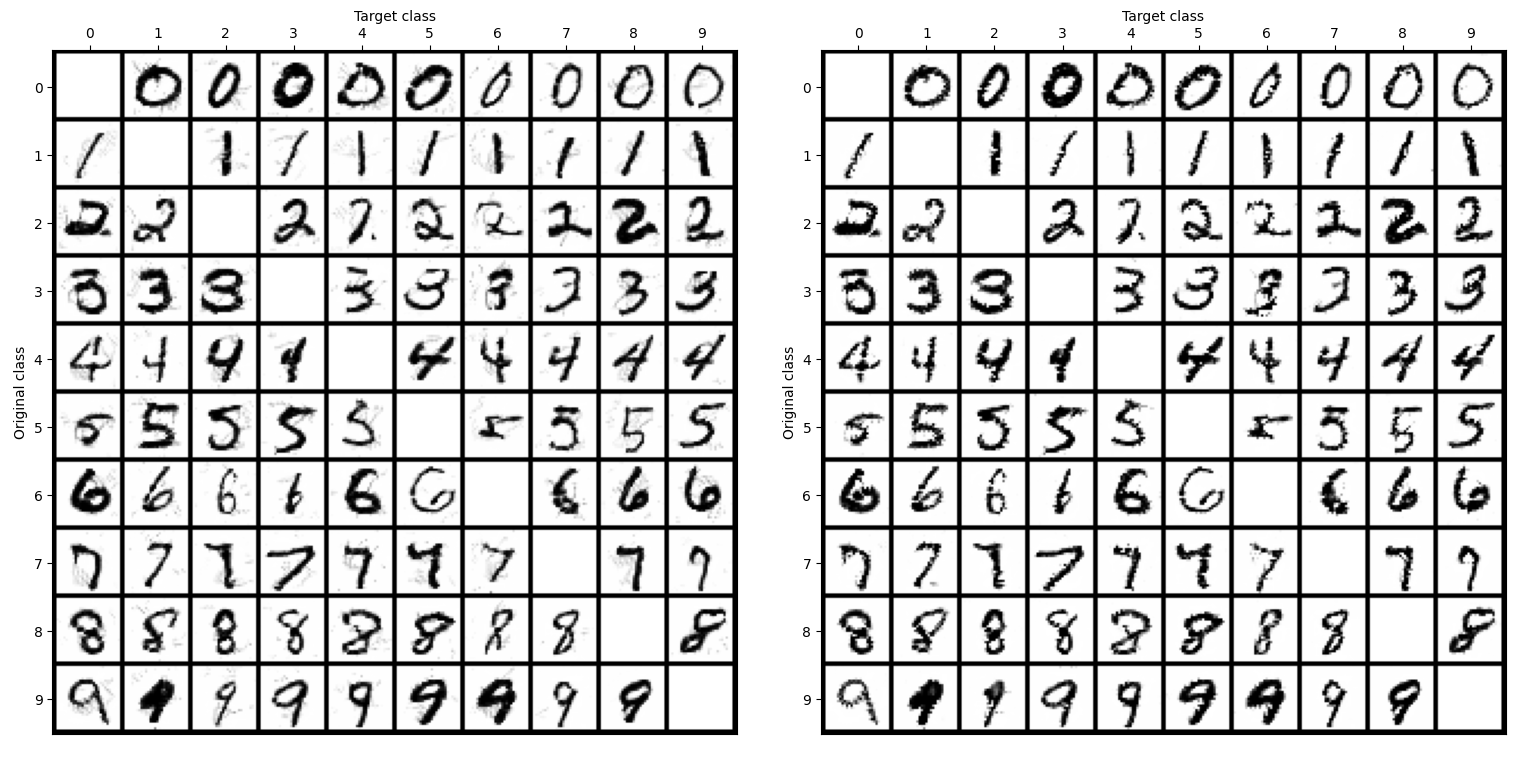}
\caption{Targeted $l_2$ attacks (left) and Wasserstein attacks (right). Rows represent the original while columns the target class.}
\label{fig:compare_matrix}
\end{figure*}

For the numerical experiments, we use the notation of $l_2$ and Wasserstein attacks based on the objective function in the model \eqref{eq:problem}. For MNIST, we would like to stress that none of the images was manually selected, and all images were generated randomly.

Figure \ref{fig:compare_matrix} shows targeted $l_2$ attacks (left) and targeted Wasserstein attacks (right). Even though both attacks performed well, the Wasserstein attack shows fewer grey artefacts around the digits. This is natural as the Wasserstein distance considers the spatial distribution of pixels. The is not the case for the $l_2$ distance.

Table \ref{table:results} shows a numerical comparison between these two attacks and the CW $l_2$ attack \cite{carlini2017evaluating} implemented in the Foolbox library \cite{rauber2017foolboxnative}. Besides the targeted attacks, we also implemented untargeted attacks and attacks against the robust network (additional figures are in the appendix). We show the $l_2$ and the Wasserstein distances. The table shows that if we optimize with respect to the Wasserstein distance, the Wasserstein distance between the original and adversarial images (column 8) is the smallest. The same holds for the $l_2$ distance (column 4). Even though the CW attack generated smaller values in the $l_2$ distance, this happened because it is not restricted by the decoder. Moreover, as we will see from the next figure, the CW attack generates lower-quality images which do not keep the hidden structure of the original images. It is also not surprising that the needed perturbations are smaller for untargeted attacks and the non-robust network.

\begin{figure*}[!ht]
    \centering
    \includegraphics[width=0.9\textwidth]{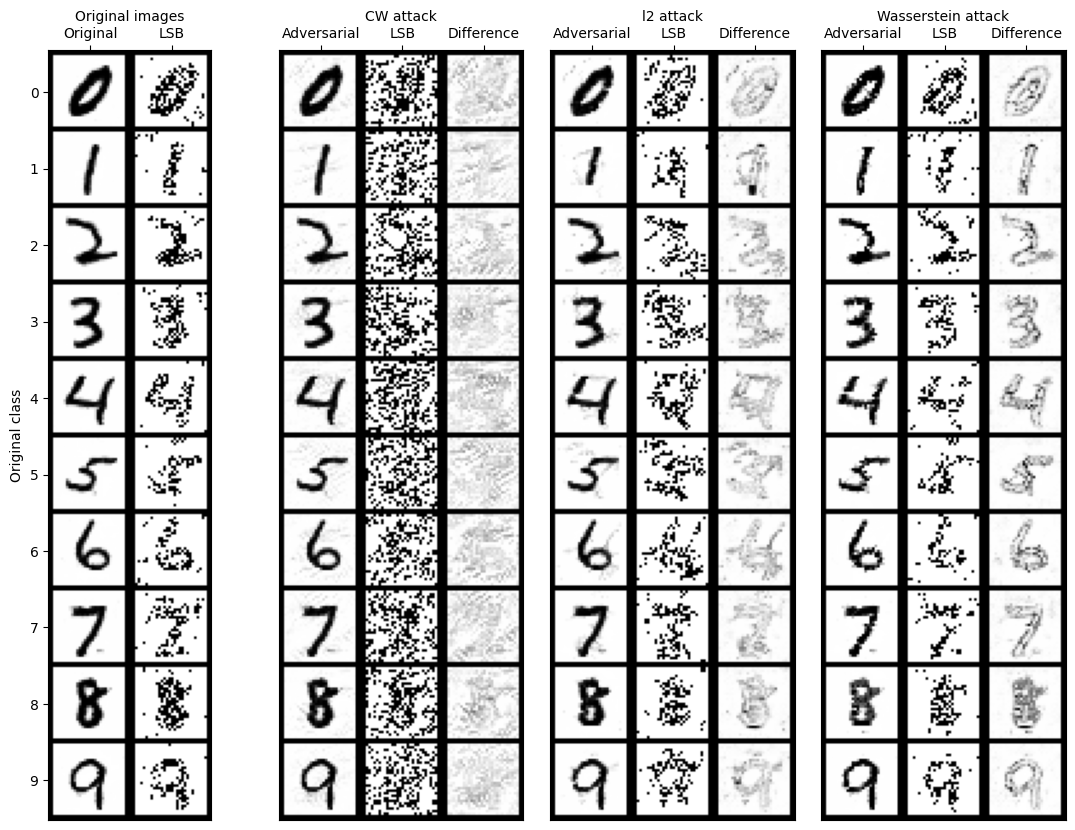}
\caption{Comparison of CW and our attacks. Least significant bit (LSB) is defined in \eqref{eq:lsb}, the difference is between the original and adversarial images. The figure shows that out attacks keep the hidden structure of the data (LSB) and that the perturbations are not located randomly but around edges.}
\label{fig:compare_attacks}
\end{figure*}

Figure \ref{fig:compare_attacks} shows a visual comparison for untargeted attacks. It shows the original image (left) and the CW (middle left), $l_2$ (middle right) and Wasserstein (right) attacks. Each three columns contain the adversarial image, the least significant bit representation and the difference between the original and adversarial images. Least significant bit is a standard steganographic defence technique which captures changes in the hidden structure of the data. For pixel intensities in $x\in [0,1]$, it is defined by
\begin{equation}\label{eq:lsb}
\operatorname{lsb}(x) = \operatorname{mod}(\operatorname{round}(255x), 2).
\end{equation}
Therefore, it transforms the image $x\in[0,1]$ into its $8$-bit representation and takes the last (least significant) bit.

A significant difference between our attacks and the pixel-based CW attack is that our attacks provide interpretability of perturbations. Figure~\ref{fig:compare_attacks} shows that there is no clear pattern in the pixel perturbations of the CW attack, neither in the least significant bit nor in the (almost) uniformly distributed perturbations. On the other hand, our method keeps the least significant bit structure. This implies that the least significant bit defence can easily recognize the CW attack, while our attacks cannot be recognized. At the same time, our methods concentrate the attack around the edges of the image. Therefore, our attacks are compatible with the basic steganographic rule stating that attacks should be concentrated mainly around key features such as edges. We show the same figure for the robust classifier in the appendix.

\subsection{Numerical results for Imagenet}

When fooling the ImageNet classifier, we use only the targeted version of our algorithm to prevent trivial class changes in the case of the untargeted attack, such as changing the dog's breed. We manually select several pictures of various animals as original images. As the target class we chose broccoli due to its distinct features.

Figure~\ref{fig:imagenet_results} demonstrates the differences between original and adversarial images. In all experiments, we perturbed the second intermediate layer of the BigBiGAN decoder. The first row shows the original images, while the second row shows the adversarial images misclassified all as broccoli. The third row highlights the difference between original and adversarial images by taking the mean square root of the absolute pixel difference across all channels. This difference shows interpretable silhouettes, which allows us to assign semantical meaning to many perturbations. Most of the differences are concentrated in key components of the animals, such as changes of texture, colouring and brightness of their semantically meaningful components. For example, the dog fur has a different texture, and the grass has lighter colour. Another example is the bird's beak, which is longer in the third and thicker in the fourth adversarial image. The difference plot also shows that significant perturbations happen around animals edges or key features such as eyes and nose.

\begin{figure}[!ht]
    \includegraphics[width=\linewidth]{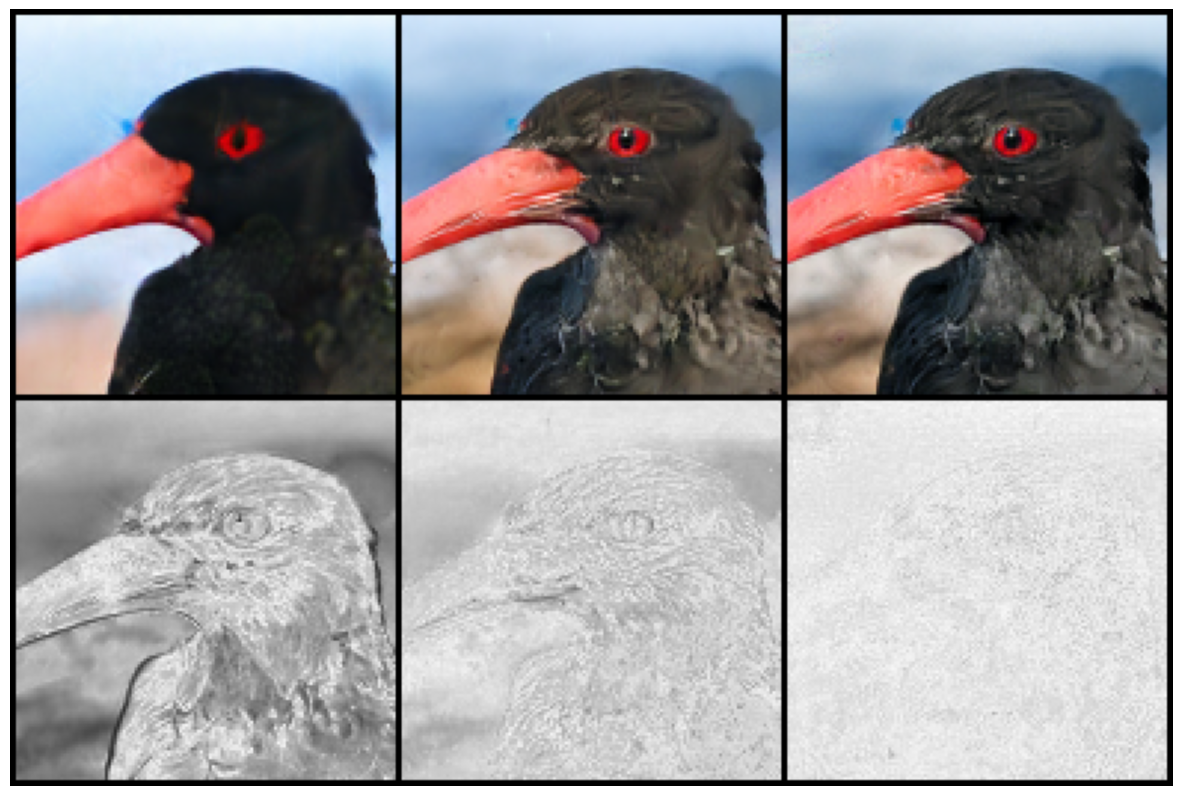}
\caption{Effect of the intermediate layer choice. The columns show the perturbation in the first (left), second (middle) and third (right) intermediate layer.}
\label{fig:imagenet_intermediate}
\end{figure}

\begin{figure*}[!ht]
 \centering
    \includegraphics[width=0.8\textwidth]{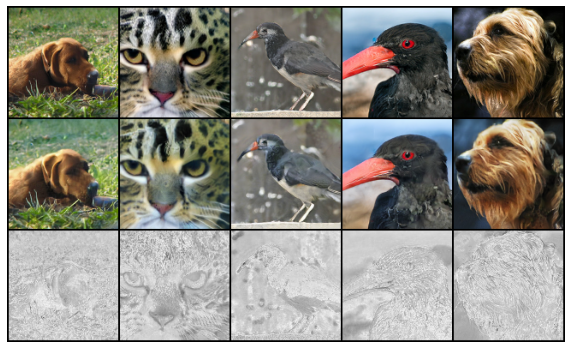}
\caption{Unperturbed image (top), perturbed image (middle) and their difference (bottom). The perturbations happen around key animal features.}
\label{fig:imagenet_results}
\end{figure*}

\begin{figure*}[!ht]
    \includegraphics[width=\textwidth]{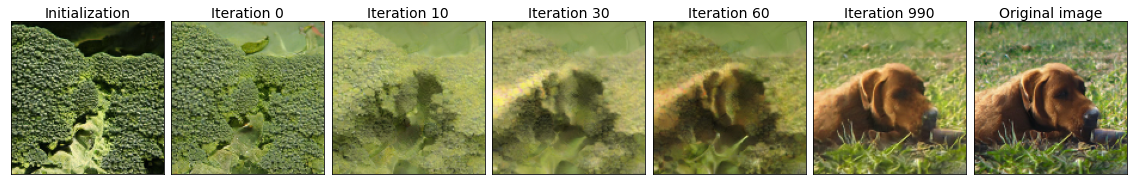}
\caption{Development of perturbed images over iterations. Broccoli features get gradually incorporated into the dog.}
\label{fig:imagenet_development}
\end{figure*}

The crucial idea of our paper is to perturb the intermediate decoder layers. However, we can place the perturbation in any intermediate layer, each with a different effect. The earlier layers generate high-level features, while later layers generate much finer features or even perturbations in pixels without any semantic meaning. Figure \ref{fig:imagenet_intermediate} shows the impact of the choice of the intermediate layer. The columns show the adversarial images when perturbing the first (left column), second (middle column) and third (right column) intermediate layer. The figure shows that the deeper we perturb the decoder, the more the perturbations shift from high-level to finer perturbations. The first layer results in a purely black bird and uniform background. The silhouettes show that while the first layer perturbs the colour of the chest,  the beak or the background, the silhouette in the third layer disappears, and the perturbations amount almost to pixel perturbations. We present additional results of perturbing other layers in the appendix.

Figure~\ref{fig:imagenet_development} documents a similar effect of high-level perturbations. While the previous figure showed the dependence of high-level features on the intermediate layer choice, this figure shows this dependence on the development of iterations in Algorithm \ref{alg2}. The left image shows the broccoli image used for the initialization of our algorithm. The right image shows the unperturbed dog. The middle five images visualize how the perturbations modify the broccoli image as the number of iterations increases. At the beginning of the algorithm run, the perturbations are interpretable: the dog is reconstructed by high-level broccoli features. As the algorithm gradually converges, the perturbed image increasingly resembles the original dog as the broccoli features fuse with the dog image. Some of the original broccoli features remain integrated into the final adversarial image. For example, the next-to-last picture contains a slightly lighter green shade than the original dog picture. Similarly, we can interpret the changes in the dog fur and muzzle textures as they originate from the broccoli texture. We would like to point out that all central images are classified as broccoli due to the construction of our algorithm. We show the impact of the choice of the initial broccoli image on the final adversarial image in the appendix.

\section{Conclusion}

This paper presented a novel method for generating adversarial images. Instead of the standard pixel perturbation, we use an encoder-decoder model and perturb high-level features from intermediate decoder layers. We showed that our method generates high-quality adversarial images in both targeted and untargeted settings on the MNIST and ImageNet datasets. We demonstrated that our adversarial images indeed perturb high-level features. This makes them more resilient to being recognized as adversarial by standard defence techniques.


\bibliography{bibliography}

\begin{thebibliography}{25}
\providecommand{\natexlab}[1]{#1}
\providecommand{\url}[1]{\texttt{#1}}
\providecommand{\urlprefix}{URL }
\expandafter\ifx\csname urlstyle\endcsname\relax
  \providecommand{\doi}[1]{doi:\discretionary{}{}{}#1}\else
  \providecommand{\doi}{doi:\discretionary{}{}{}\begingroup
  \urlstyle{rm}\Url}\fi

\bibitem[{{Brendel}, {Rauber}, and {Bethge}(2018)}]{brendel2018decisionbased}
{Brendel}, W.; {Rauber}, J.; and {Bethge}, M. 2018.
\newblock Decision-Based Adversarial Attacks: Reliable Attacks Against
  Black-Box Machine Learning Models.
\newblock In \emph{International Conference on Learning Representations (ICLR)
  2018}.

\bibitem[{{Carlini} and {Wagner}(2017)}]{carlini2017evaluating}
{Carlini}, N.; and {Wagner}, D. 2017.
\newblock Towards Evaluating the Robustness of Neural Networks.
\newblock In \emph{2017 IEEE Symposium on Security and Privacy (SP)}, 39--57.

\bibitem[{{Chen}, {Jordan}, and {Wainwright}(2020)}]{chen2020hopskipjumpattack}
{Chen}, J.; {Jordan}, M.~I.; and {Wainwright}, M.~J. 2020.
\newblock HopSkipJumpAttack: A Query-Efficient Decision-Based Attack.
\newblock In \emph{2020 IEEE Symposium on Security and Privacy (SP)},
  1277--1294.

\bibitem[{{Cuturi}(2013)}]{cuturi2013sinkhorn}
{Cuturi}, M. 2013.
\newblock Sinkhorn Distances: Lightspeed Computation of Optimal Transport.
\newblock In \emph{Advances in Neural Information Processing Systems 26},
  volume~26, 2292--2300.

\bibitem[{Deng et~al.(2009)Deng, Dong, Socher, Li, Li, and
  Fei-Fei}]{deng2009imagenet}
Deng, J.; Dong, W.; Socher, R.; Li, L.-J.; Li, K.; and Fei-Fei, L. 2009.
\newblock Imagenet: A large-scale hierarchical image database.
\newblock In \emph{2009 IEEE conference on computer vision and pattern
  recognition}, 248--255.

\bibitem[{Donahue, Kr{\"a}henb{\"u}hl, and
  Darrell(2016)}]{donahue2017adversarial}
Donahue, J.; Kr{\"a}henb{\"u}hl, P.; and Darrell, T. 2016.
\newblock Adversarial feature learning.
\newblock \emph{arXiv preprint arXiv:1605.09782} .

\bibitem[{{Donahue} and {Simonyan}(2019)}]{donahue2019large}
{Donahue}, J.; and {Simonyan}, K. 2019.
\newblock Large Scale Adversarial Representation Learning.
\newblock In \emph{Advances in Neural Information Processing Systems},
  volume~32, 10541--10551.

\bibitem[{Dumoulin et~al.(2016)Dumoulin, Belghazi, Poole, Mastropietro, Lamb,
  Arjovsky, and Courville}]{dumoulin2017adversarially}
Dumoulin, V.; Belghazi, I.; Poole, B.; Mastropietro, O.; Lamb, A.; Arjovsky,
  M.; and Courville, A. 2016.
\newblock Adversarially learned inference.
\newblock \emph{arXiv preprint arXiv:1606.00704} .

\bibitem[{Feydy et~al.(2019)Feydy, S{\'e}journ{\'e}, Vialard, Amari, Trouve,
  and Peyr{\'e}}]{feydy2019interpolating}
Feydy, J.; S{\'e}journ{\'e}, T.; Vialard, F.-X.; Amari, S.-i.; Trouve, A.; and
  Peyr{\'e}, G. 2019.
\newblock Interpolating between Optimal Transport and MMD using Sinkhorn
  Divergences.
\newblock In \emph{The 22nd International Conference on Artificial Intelligence
  and Statistics}, 2681--2690.

\bibitem[{{Goodfellow}, {Shlens}, and
  {Szegedy}(2015)}]{goodfellow2015explaining}
{Goodfellow}, I.~J.; {Shlens}, J.; and {Szegedy}, C. 2015.
\newblock Explaining and Harnessing Adversarial Examples.
\newblock In \emph{International Conference on Learning Representations (ICLR)
  2015}.

\bibitem[{Johnson and Jajodia(1998)}]{johnson1998exploring}
Johnson, N.~F.; and Jajodia, S. 1998.
\newblock Exploring steganography: Seeing the unseen.
\newblock \emph{Computer} 31(2): 26--34.

\bibitem[{Kurakin, Goodfellow, and Bengio(2016)}]{kurakin2017adversarial}
Kurakin, A.; Goodfellow, I.; and Bengio, S. 2016.
\newblock Adversarial machine learning at scale.
\newblock \emph{arXiv preprint arXiv:1611.01236} .

\bibitem[{LeCun, Cortes, and Burges(2010)}]{lecun2010mnist}
LeCun, Y.; Cortes, C.; and Burges, C. 2010.
\newblock MNIST handwritten digit database.
\newblock \emph{ATT Labs [Online]. Available: http://yann.lecun.com/exdb/mnist}
  2.

\bibitem[{{Madry} et~al.(2018){Madry}, {Makelov}, {Schmidt}, {Tsipras}, and
  {Vladu}}]{madry2019deep}
{Madry}, A.; {Makelov}, A.; {Schmidt}, L.; {Tsipras}, D.; and {Vladu}, A. 2018.
\newblock Towards Deep Learning Models Resistant to Adversarial Attacks.
\newblock In \emph{International Conference on Learning Representations (ICLR)
  2018}.

\bibitem[{{Meng} and {Chen}(2017)}]{meng2017magnet}
{Meng}, D.; and {Chen}, H. 2017.
\newblock MagNet: A Two-Pronged Defense against Adversarial Examples.
\newblock In \emph{Proceedings of the 2017 ACM SIGSAC Conference on Computer
  and Communications Security}, 135--147.

\bibitem[{Nocedal and Wright(2006)}]{nocedal2006numerical}
Nocedal, J.; and Wright, S. 2006.
\newblock \emph{Numerical optimization}.
\newblock Springer Science \& Business Media.

\bibitem[{{Papernot} et~al.(2016){Papernot}, {McDaniel}, {Jha}, {Fredrikson},
  {Celik}, and {Swami}}]{papernot2015limitations}
{Papernot}, N.; {McDaniel}, P.; {Jha}, S.; {Fredrikson}, M.; {Celik}, Z.~B.;
  and {Swami}, A. 2016.
\newblock The Limitations of Deep Learning in Adversarial Settings.
\newblock In \emph{2016 IEEE European Symposium on Security and Privacy
  (EuroS\&P)}, 372--387.

\bibitem[{Rauber et~al.(2020)Rauber, Zimmermann, Bethge, and
  Brendel}]{rauber2017foolboxnative}
Rauber, J.; Zimmermann, R.; Bethge, M.; and Brendel, W. 2020.
\newblock {Foolbox Native: Fast adversarial attacks to benchmark the robustness
  of machine learning models in PyTorch, TensorFlow, and JAX}.
\newblock \emph{Journal of Open Source Software} 5(53): 2607.
\newblock \doi{10.21105/joss.02607}.
\newblock \urlprefix\url{https://doi.org/10.21105/joss.02607}.

\bibitem[{{Samangouei}, {Kabkab}, and
  {Chellappa}(2018)}]{samangouei2018defensegan}
{Samangouei}, P.; {Kabkab}, M.; and {Chellappa}, R. 2018.
\newblock Defense-GAN: Protecting Classifiers Against Adversarial Attacks Using
  Generative Models.
\newblock In \emph{International Conference on Learning Representations (ICLR)
  2018}.

\bibitem[{{Simonyan} and {Zisserman}(2015)}]{simonyan2015deep}
{Simonyan}, K.; and {Zisserman}, A. 2015.
\newblock Very Deep Convolutional Networks for Large-Scale Image Recognition.
\newblock In \emph{International Conference on Learning Representations (ICLR)
  2015}.

\bibitem[{{Song} et~al.(2018){Song}, {Shu}, {Kushman}, and
  {Ermon}}]{song2018constructing}
{Song}, Y.; {Shu}, R.; {Kushman}, N.; and {Ermon}, S. 2018.
\newblock Constructing Unrestricted Adversarial Examples with Generative
  Models.
\newblock In \emph{Advances in Neural Information Processing Systems},
  volume~31, 8312--8323.

\bibitem[{{Szegedy} et~al.(2014){Szegedy}, {Zaremba}, {Sutskever}, {Bruna},
  {Erhan}, {Goodfellow}, and {Fergus}}]{szegedy2014intriguing}
{Szegedy}, C.; {Zaremba}, W.; {Sutskever}, I.; {Bruna}, J.; {Erhan}, D.;
  {Goodfellow}, I.; and {Fergus}, R. 2014.
\newblock Intriguing properties of neural networks.
\newblock In \emph{International Conference on Learning Representations (ICLR)
  2014}.

\bibitem[{{Tan} and {Le}(2019)}]{tan2020efficientnet}
{Tan}, M.; and {Le}, Q.~V. 2019.
\newblock EfficientNet: Rethinking Model Scaling for Convolutional Neural
  Networks.
\newblock In \emph{International Conference on Machine Learning}, 6105--6114.

\bibitem[{{Wong}, {Schmidt}, and {Kolter}(2019)}]{wong2020wasserstein}
{Wong}, E.; {Schmidt}, F.~R.; and {Kolter}, J.~Z. 2019.
\newblock Wasserstein Adversarial Examples via Projected Sinkhorn Iterations.
\newblock In \emph{International Conference on Machine Learning}, 6808--6817.

\bibitem[{{Xiao} et~al.(2018){Xiao}, {Li}, yan {Zhu}, {He}, {Liu}, and
  {Song}}]{xiao2019generating}
{Xiao}, C.; {Li}, B.; yan {Zhu}, J.; {He}, W.; {Liu}, M.; and {Song}, D. 2018.
\newblock Generating Adversarial Examples with Adversarial Networks.
\newblock In \emph{Proceedings of the Twenty-Seventh International Joint
  Conference on Artificial Intelligence}, 3905--3911.

\end{thebibliography}


\appendix

The Appendix presents the proof of Theorem \ref{thm:conv} and shows additional results.

\section{Proof of Theorem \ref{thm:conv}}

Recall that a function $f$ is Lipschitz continuous with constant $L$ if
$$
\|f(x)-f(y)\| \le L\|x-y\|
$$
for all $x$ and $y$.

\begin{theorem}
    Assume that $g\circ D_2$ is a continuous function. Then Algorithm \ref{alg1} generates a strictly feasible point $p$ in a finite number of iterations. If $g\circ D_2$ is moreover Lipschitz continuous with modulus $L$, then Algorithm \ref{alg1} converges in at most $\log_2 \frac{L\|p\|}{\delta}$ iterations.
\end{theorem}
\begin{proof}
    Denote the iterations from Algorithm \ref{alg1} by $b^k$ and $a^k$ with the initialization $b^0=1$ and $a^0=0$. Since the interval halves at each iteration, we obtain
    \begin{equation}\label{eq:conv1}
    b^k - a^k = \frac12(b^{k-1}-a^{k-1}) = \frac{1}{2^k}(b^0 - a^0) = \frac{1}{2^k}.
    \end{equation}
    Define
    $$
    \tilde g(c) = \hat g((1-c)p + c\pnext).
    $$
    Since we have
    $$
    \aligned
    \tilde g(b^0) &= \tilde g(1) = \hat g(\pnext) \le 0, \\
    \tilde g(a^0) &= \tilde g(0) = \hat g(p) < 0.
    \endaligned
    $$
    due to the construction of the algorithm, we have $\tilde g(b^k) \le 0$ and $\tilde g(a^k) > 0$ for all $k$. Since $\tilde g$ is a continuous function due to continuity of $D_2\circ g$, Algorithm \ref{alg1} converges in a finite number of iterations.
    
    Assume that $g\circ D_2$ is a Lipschitz continuous function with modulus $L$, then $\hat g$ is also Lipschitz continuous with modulus $L$. Then
    $$
    \aligned
    |\tilde g&(c_1) - \tilde g(c_2)| \\
    &= |\hat g((1-c_1)p + c_1\pnext) - \hat g((1-c_2)p + c_2\pnext)| \\
    &\le L\| (1-c_1)p +c_1\pnext - (1-c_2)p - c_2\pnext\| \\
    &= L\| (\pnext - p)(c_1-c_2) \| \\
    &\le L\|\pnext - p\| |c_1 - c_2 | \\
    \endaligned
    $$
    Therefore, $\tilde g$ is a Lipschitz continuous function with constant $L\|\pnext - p\|$. Together with \eqref{eq:conv1} this implies
    \begin{equation}\label{eq:conv2}
    \aligned
    \tilde g(b^k) - \tilde g(a^k) &\le L\|\pnext-p\|(b^k - a^k) \\
    &= L\|\pnext-p\|\frac{1}{2^k}.
    \endaligned
    \end{equation}
    If the algorithm did not finish within $k$ iterations, we have
    \begin{equation}\label{eq:conv3}
    \tilde g(b^k) - \tilde g(a^k) \ge \delta.
    \end{equation}
    The comparison of \eqref{eq:conv2} and \eqref{eq:conv3} implies that the algorithm cannot run for more than $\log_2 \frac{L\|\pnext-p\|}{\delta}$ iterations, which finishes the proof.
\end{proof}

\section{Additional results}

This section extends the results from the main manuscript body. We will always present a figure and then compare it with the corresponding figure from the manuscript body. The former figures start with a letter while the latter figures with a digit.

Figure \ref{fig:compare_matrix2} shows the untargeted attacks for the non-robust network. It corresponds to Figure \ref{fig:compare_matrix} from the manuscript body. The images are again nice, with the Wasserstein attack performing better than the $l_2$ attack. The small digit in each subfigure shows to which class the digit was misclassified. As we have already mentioned, our method always works with feasible points and, therefore, all digits were successfully misclassified. In other words, these images were generated randomly without the need to select them manually.

Figure \ref{fig:compare_attacks2} shows the attacks on the robust network of \cite{madry2019deep}. It corresponds to Figure \ref{fig:compare_attacks} from the main manuscript body. The quality of the CW attack increased. It now preserves the least significant bit structure, and the perturbations shifted towards the digit edges. The Wasserstein attack keeps the superb performance with visually the same results as in Figure \ref{fig:compare_attacks}. We conclude that our attacks are efficient against adversarially trained networks.

\begin{figure*}[!ht]
    \centering
    \begin{minipage}{0.48\textwidth}
    \includegraphics[width=\textwidth]{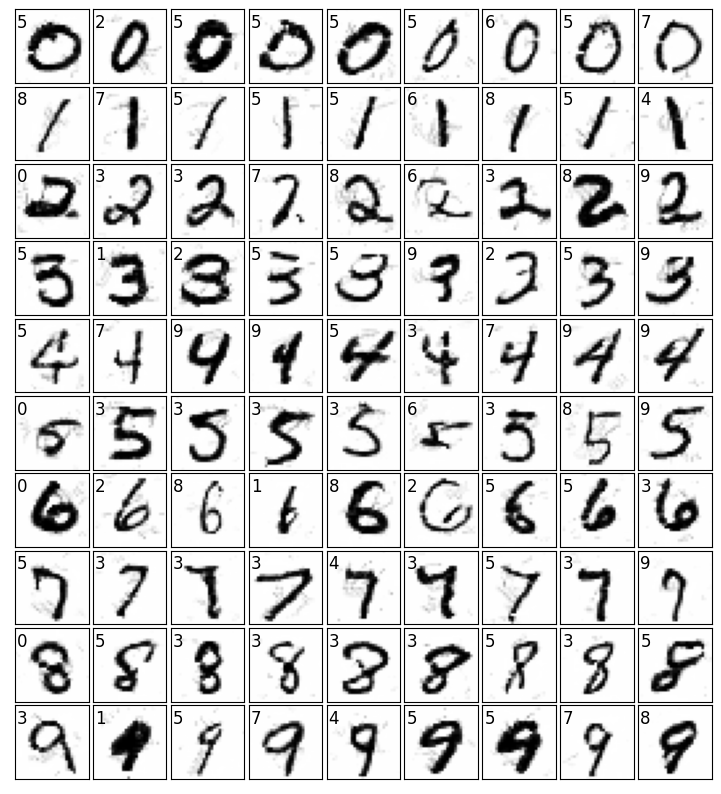}
    \end{minipage}
    \begin{minipage}{0.48\textwidth}
    \includegraphics[width=\textwidth]{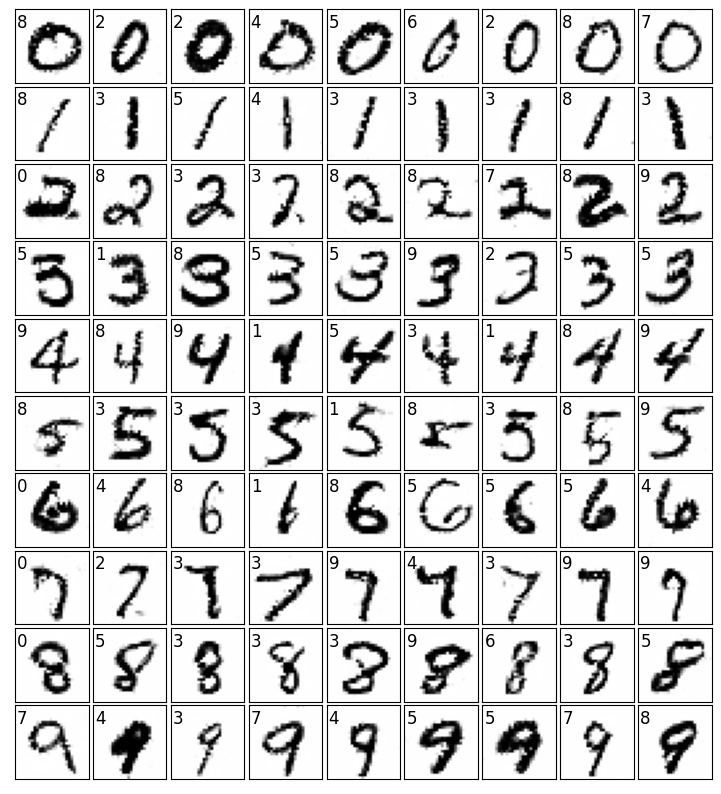}
    \end{minipage}
\caption{Untargeted $l_2$ attacks (left) and Wasserstein attacks (right). Rows represent the class. The small number in each subfigure shows to which class the digit was classified.}
\label{fig:compare_matrix2}
\end{figure*}

Figure \ref{fig:imagenet_intermediate2} shows the effect of the choices of the intermediate layer to perturb and of the initial point for Algorithm~\ref{alg2}. It corresponds to Figure \ref{fig:imagenet_intermediate} from the main manuscript body. The first column shows the point which was used to initialize Algorithm \ref{alg2}. The next three columns present the adversarial images with the perturbation inserted in different intermediate layers. The last column shows the original image. The effect of the initialization is negligible when perturbing the third intermediate layer (column 4). However, it has a huge impact when perturbing earlier intermediate layers. The intermediate latent representation of the second broccoli (rows 2 and~4) carries a preference for the white colour, which is visible in the adversarial image for the first intermediate layer (column 2).

Figure \ref{fig:imagenet_development2} also shows the effect of the initial point for Algorithm~\ref{alg2}. It corresponds to Figure \ref{fig:imagenet_development} from the main manuscript body. We see that the features of the initial broccoli (column 1) gradually incorporate into the dog image (columns 2-5). The adversarial image (column 6) still have some connection to the initial broccoli, for example, in the background colour. The adversarial image is close to the original image (column 7). This figure shows once again that our algorithm works with high-level features and not pixel modifications.

\begin{figure*}[!ht]
    \centering
    \includegraphics[width=\textwidth]{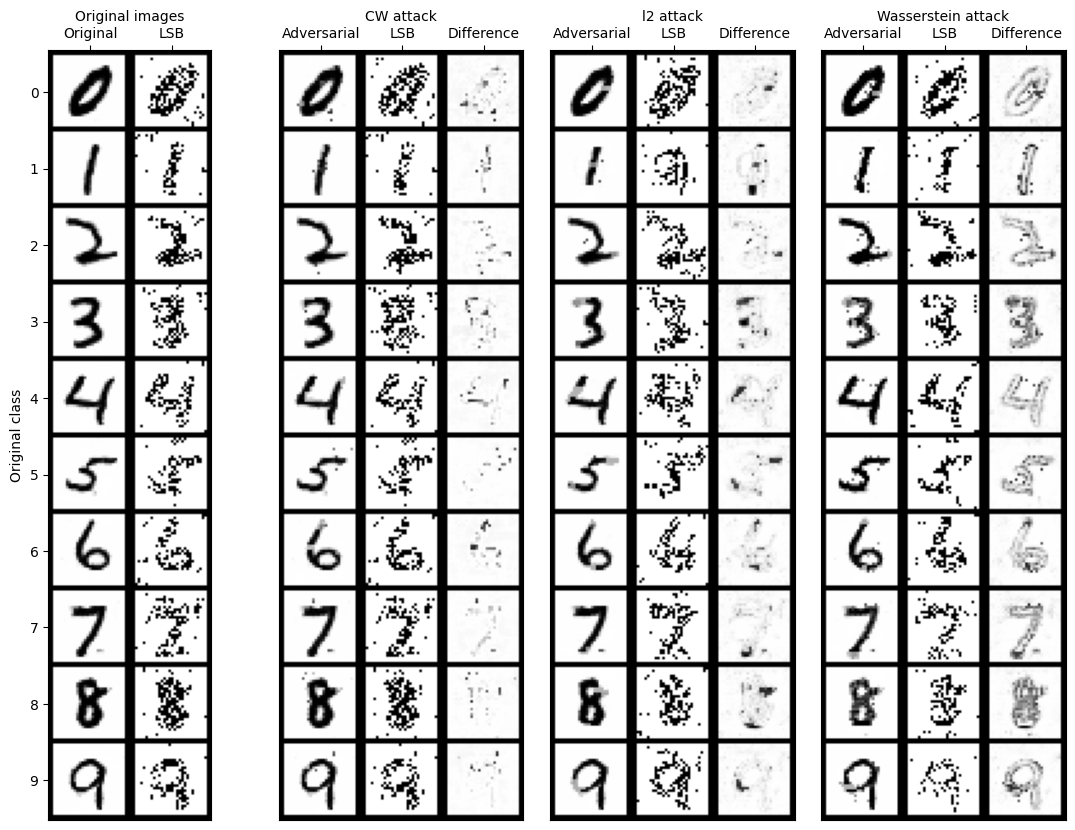}
\caption{Comparison of CW and our attacks on the robust network. Least significant bit (LSB) is defined in \eqref{eq:lsb}, the difference is between the original and adversarial images.}
\label{fig:compare_attacks2}
\end{figure*}

\begin{figure*}[!ht]
    \centering
    \includegraphics[width=0.95\textwidth]{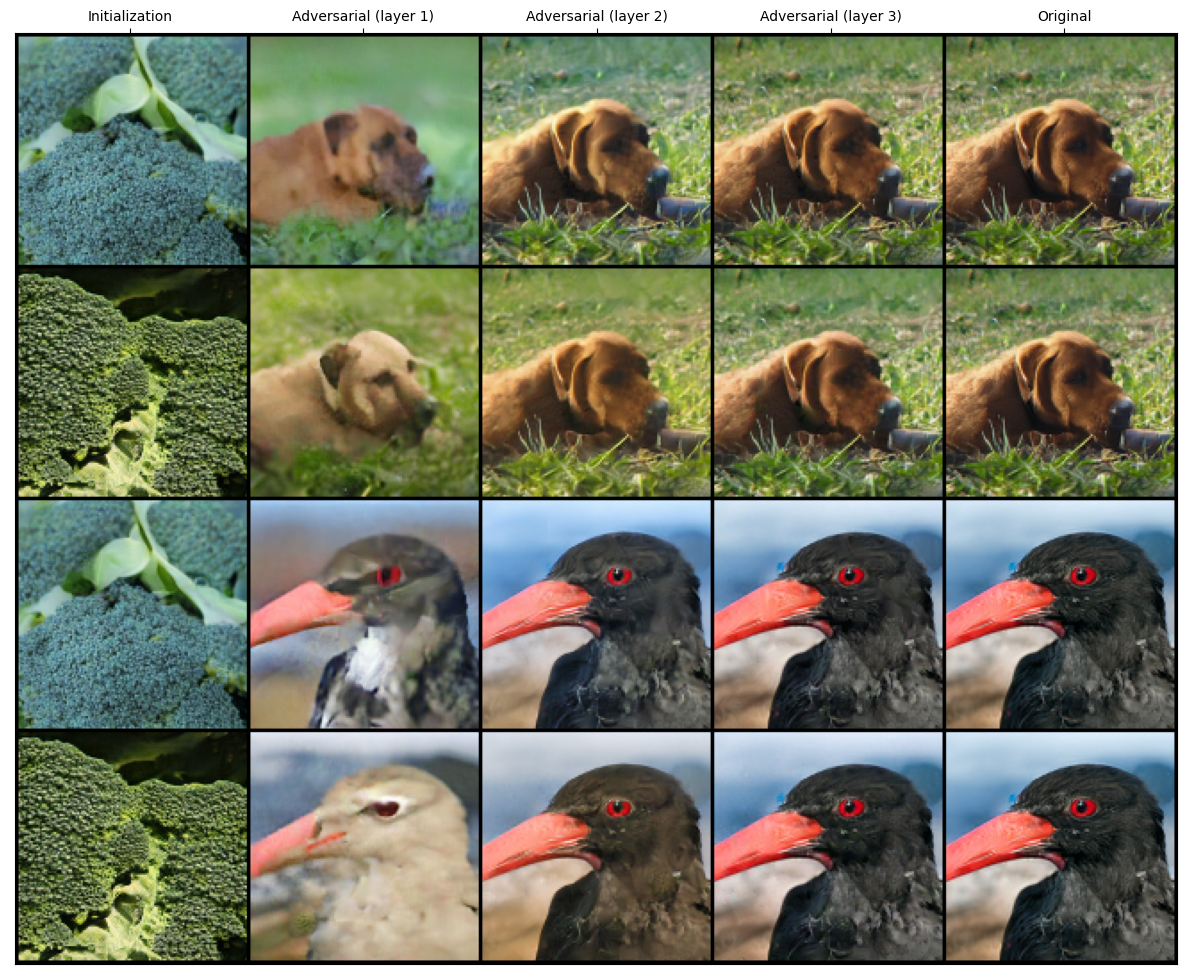}
\caption{Effect of the intermediate layer choice. The columns show initial point for Algorithm \ref{alg2} (column 1), the adversarial image when perturbations were performed in the first (column 2), second (column 3) and third (column 4) intermediate layer and the original image (column 5).}
\label{fig:imagenet_intermediate2}
\end{figure*}

\begin{figure*}[!ht]
    \centering
    \includegraphics[width=\textwidth]{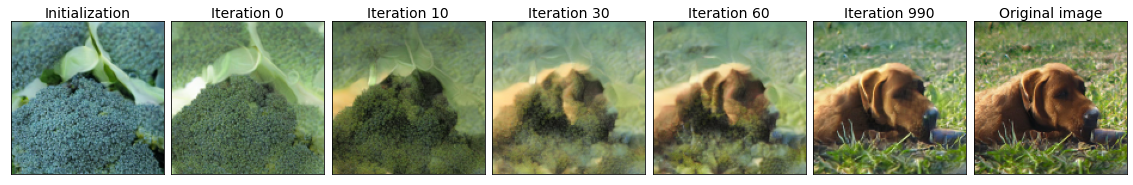}
    \includegraphics[width=\textwidth]{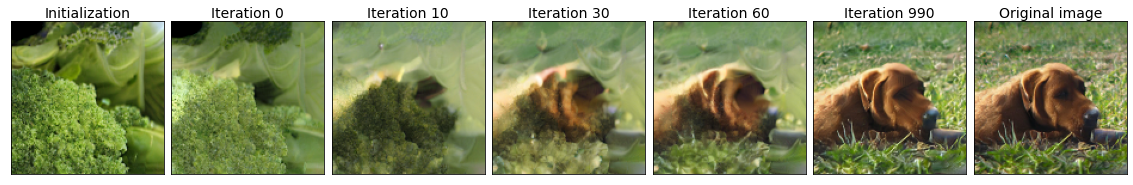}
\caption{Development of perturbed images over iterations. Broccoli features get gradually incorporated into the dog.}
\label{fig:imagenet_development2}
\end{figure*}

\end{document}